\documentclass{article}
\usepackage{arxiv}

\usepackage{times}

\usepackage{microtype}
\usepackage{graphicx}
\usepackage{subfigure}
\usepackage{booktabs} 
\usepackage{xcolor}
\usepackage{paralist}

\usepackage{natbib}
\usepackage{epstopdf}
\usepackage{algorithm}
\usepackage[noend]{algorithmic}

\usepackage{amssymb}
\usepackage{amsmath}
\usepackage{amsfonts}
\usepackage{array}
\usepackage[draft]{hyperref}
\usepackage{tabularx}

\usepackage[inline]{enumitem}
\usepackage{dL}
\usepackage{logic}
\usepackage{listings}
\usepackage{xspace}
\usepackage{subfigure}
\usepackage{algorithm}
\usepackage{textcomp}
\usepackage{stmaryrd}
\usepackage{amsthm}
\usepackage{todonotes}

\usepackage{prettyref}
\newcommand{\rref}[2][]{\prettyref{#2}}
\newrefformat{model}{Model\,\ref{#1}}
\newrefformat{listing}{Listing\,\ref{#1}}
\newrefformat{alg}{Algorithm\,\ref{#1}}
\newrefformat{line}{line\,\ref{#1}}
\newrefformat{sec}{Section\,\ref{#1}}
\newrefformat{appendix}{Appendix\,\ref{#1}}
\newrefformat{app}{Appendix\,\ref{#1}}
\newrefformat{def}{Definition\,\ref{#1}}
\newrefformat{thm}{Theorem\,\ref{#1}}
\newrefformat{ax}{\ref{#1}}
\newrefformat{prop}{Proposition\,\ref{#1}}
\newrefformat{lemma}{Lemma\,\ref{#1}}
\newrefformat{cor}{Corollary\,\ref{#1}}
\newrefformat{ex}{Example\,\ref{#1}}
\newrefformat{tab}{Table\,\ref{#1}}
\newrefformat{fig}{Figure\,\ref{#1}}
\newrefformat{eqn}{Equation~(\ref{#1})}
\newrefformat{problem}{Problem\,\ref{#1}}

\newtheorem{example}{Example}
\newtheorem{cor}{Corollary}

\newtheorem{theorem}{Theorem}
\newtheorem{theorem*}{Theorem}
\newtheorem{lemma}{Lemma}

\newtheorem{assumption}{Assumption}

\usepackage{mathrsfs}
\newcommand{\mname}{\texttt{VSRL}\xspace}

\newcommand{\den}[1]{\llbracket#1\rrbracket}
\newcommand{\assignm}{\leftarrow}
\newcommand{\assign}{\(\assignm\)}
\newcommand{\sinit}{\mathcal{S}_{init}}
\newcommand{\mc}{\mathcal}
\newcommand{\E}[2][]{\mathop{\mathbb{E}_{#1}} \left[ #2 \right]}

\begin{document}

\title{Verifiably Safe Exploration for End-to-End Reinforcement Learning}

\author{%
    Nathan Hunt$^1$, Nathan Fulton$^2$, Sara Magliacane$^2$, Nghia Hoang$^2$, Subhro Das$^2$, Armando Solar-Lezama$^1$\thanks{The authors acknowledges support from the MIT-IBM Watson AI Lab. The email addresses of the authors are: nhunt@mit.edu, \{nathan, sara.magliacane, nghiaht, subhro.das\}@ibm.com, asolar@csail.mit.edu. } \\
    $^1$ Massachusetts Institute of Technology \\ 
    $^2$ MIT-IBM Watson AI Lab, IBM Research
}

\maketitle

\begin{abstract}
Deploying deep reinforcement learning in safety-critical settings requires developing algorithms that obey hard constraints during exploration.
This paper contributes a first approach toward enforcing formal safety constraints on end-to-end policies with visual inputs.
Our approach draws on recent advances in object detection and automated reasoning for hybrid dynamical systems.
The approach is evaluated on a novel benchmark that emphasizes the challenge of safely exploring in the presence of hard constraints. Our benchmark draws from several proposed problem sets for safe learning and includes problems that emphasize challenges such as 
reward signals that are not aligned with safety constraints.
On each of these benchmark problems, our algorithm completely avoids unsafe behavior while remaining competitive at optimizing for as much reward as is safe. 
We also prove that our method of enforcing the safety constraints preserves all safe policies from the original environment. 
\end{abstract}

\section{Introduction}
\label{sec:intro}
Deep reinforcement learning algorithms \citep{sutton.barto:reinforcement} are effective at learning, often from raw sensor inputs, control policies that optimize for a quantitative reward signal.
Learning these policies can require experiencing millions of unsafe actions.
Even if a safe policy is finally learned -- which will happen only if the reward signal reflects all relevant safety priorities -- providing a purely statistical guarantee that the optimal policy is safe requires an unrealistic amount of training data \citep{RANDDriveToSafety}.
The difficulty of establishing the safety of these algorithms makes it difficult to justify the use of reinforcement learning in safety-critical domains where industry standards demand strong evidence of safety prior to deployment \citep{iso26262}. 

Formal verification provides a rigorous way of establishing safety for traditional control systems \citep{clarke2018handbook}. 
The problem of providing formal guarantees in RL is called \emph{formally constrained reinforcement learning (FCRL)}. 
Existing FCRL methods such as \citep{HasanbeigKroening,hasanbeig2018,Hasanbeig2019, hasanbeig2020cautious, DBLP:conf/tacas/HahnPSSTW19,DBLP:conf/aaai/AlshiekhBEKNT18,aaai18,DBLP:journals/corr/neuralsimplex,de2019foundations} combine the best of both worlds: they optimize for a reward function while safely exploring the environment.

Existing FCRL methods suffer from two significant disadvantages that detract from their real-world applicability: a) they enforce constraints over a completely symbolic state space that is assumed to be noiseless (e.g. the position of the safety-relevant objects are extracted from a simulator's internal state); b) they assume that the entire reward structure depends upon the same symbolic state-space used to enforce formal constraints.
The first assumption limits the applicability of FCRL in real-world settings where the system's state must be inferred by an imperfect and perhaps even untrusted perception system. The second assumption implies a richer symbolic state that includes a symbolic representation of the reward, which we argue is unnecessary and may require more labelled data. Furthermore, this means these approaches may not generalize across different environments that have similar safety concerns, but completely different reward structures.

The goal of this paper is to \emph{safely learn a safe policy} without assuming a perfect oracle that identifies the positions of all safety-relevant objects. I.e., unlike all existing FCRL methods, we do not rely on the internal state of the simulator.
Prior to reinforcement learning, we train an object detection system to extract the positions of safety-relevant objects up to a certain precision. 
The pre-trained object detection system is used during reinforcement learning to extract the positions of safety-relevant objects, and that information is then used to enforce formal safety constraints.
Absolute safety in the presence of untrusted perception is epistemologically challenging, but our formal safety constraints do at least account for a type of noise commonly found in object detection systems. 
Finally, although our system (called Verifiably Safe Reinforcement Learning, or \mname) uses a few labeled data to pre-train the object detection, we still learn an end-to-end policy that may leverage the entire visual observation for reward optimization.

Prior work from the formal methods community has demonstrated that you can do safe RL when you have full symbolic characterization of the environment and you can precisely observe the entire state. However, this is not realistic for actual robotic systems which have to interact with the physical world and can only perceive it through an imperfect visual system. This paper demonstrates that techniques inspired by formal methods can provide value even in this situation. First, we show that by using existing vision techniques to bridge between the visual input and the symbolic representation, one can leverage formal techniques to achieve highly robust behavior. Second, we prove that under weak assumptions on this vision system, the new approach will safely converge to an optimal safe policy. 

Our convergence result is the first of its kind for formally constrained reinforcement learning. Existing FCRL algorithms provide convergence guarantees only for an MDP that is defined over high-level symbolic features that are extracted from the internal state of a simulator.  Instead, we establish optimality for policies that are learned from the low-level feature space (i.e., images). We prove that our method is capable of optimizing for reward even when significant aspects of the reward structure are not extracted as high-level features used for safety checking. Our experiments demonstrate that \mname is capable of optimize for reward structure related to objects whose positions we do \emph{not} extract via supervised training. This is significant because it means that \mname needs pre-trained object detectors only objects that are safety-relevant.

Finally, we provide a novel benchmark suite for Safe Exploration in Reinforcement Learning that includes both environments where the reward signal is aligned with the safety objectives and environments where the reward-optimal policy is unsafe. Our motivation for the latter is that assuming reward-optimal policies respect hard safety constraints neglects one of the fundamental challenges of Safe RL: preventing ``reward-hacking". For example, it fundamentally difficult to tune a reward signal so that it has the ``correct" trade-off between a pedestrian's life and battery efficiency. 
We show that in the environments where the reward-optimal policy is safe (``reward-aligned''), \mname learns a safe policy with convergence rates and final rewards which are competitive or even superior to the baseline method. More importantly, \mname learns these policie with zero safety violations during training; i.e., it achieves perfectly safe exploration. In the environments where the reward-optimal policy is unsafe (``reward-misaligned''), \mname both effectively optimizes for the subset of reward that can be achieved without violating safety constraints and successfully avoids ``reward-hacking" by violating safety constraints.

Summarily, this paper contributes:
\textbf{(1)} \mname, a new approach toward formally constrained reinforcement learning that does not make unrealistic assumptions about oracle access to symbolic features. This approach requires minimal supervision before reinforcement learning begins and explores safely while remaining competitive at optimizing for reward.
\textbf{(2)} Theorems establishing that \mname learns safely and maintains convergence properties of any underlying deep RL algorithm within the set of safe policies.
\textbf{(3)} A novel benchmark suite for Safe Exploration in Reinforcement Learning that includes both properly specified and mis-specified reward signals.

\section{Problem Definition}
\label{sec:probdef}

A reinforcement learning (RL) system can be represented as a Markov Decision Process (MDP) $(\mc S, \mc A, T, R, \gamma)$ which includes a (possibly infinite) set $\mc S$ of system states, an action space $\mc A$, a transition function $T(s, a, s')$ which specifies the probability of the next system state being $s'$ after the agent executes action $a$ at state $s$, a reward function $R(s, a)$ that gives the reward for taking action $a$ in state $s$, and a discount factor $0 < \gamma < 1$ that indicates the system preference to earn reward as fast as possible. We denote the set of initial states as $\sinit \subseteq \mc S$. 

In our setting, $\mc S$ are images and we are given a safety specification $\texttt{safe}: \mc O \rightarrow \{0,1\}$ over a set of high-level observations $\mc O$, specifically, the positions (planar coordinates) of the safety-relevant objects in a 2D or 3D space.
Since $\mc S \not = \mc O$, it is not trivial to learn a safe policy $\pi$ 
such that $\texttt{safe}(\mc O) = 1$ along every trajectory.
We decompose this challenge into two well-formed and tractable sub-problems: 
\begin{enumerate}
    \item Pre-training a system $\psi : \mc S \rightarrow \mc O$ that converts the visual inputs into symbolic states using synthetic data (without acting in the environment);
    \item Learning policies over the visual input space $\mc S$ while enforcing safety in the symbolic state space $\mc O$.
\end{enumerate}
This problem is not solvable without making some assumptions, so here we focus on the following:
\begin{assumption}
\label{assm:eps}
The symbolic mapping $\psi$ is correct up to $\epsilon$. More precisely, the true position of every object $o_i$ can be extracted from the image $s$ through the object detector $\psi (s)_i$ so that the Euclidean distance between the actual and extracted positions is at most $\epsilon$, i.e. $\forall i \; ||\psi(s)_i - o_i||_2 \leq \epsilon$. 
We assume that we know an upper bound on the number of objects whose positions are extracted.
\end{assumption} 
\begin{assumption}
\label{assm:safe}
Initial states, described by a set of properties denoted as \texttt{init}, are safe,  i.e. $\forall s \in \sinit: \texttt{safe}(\psi(s))=1$ . Moreover, every state we reach after taking only safe actions has at least one available safe action.
\end{assumption} 
\begin{assumption}
\label{assm:plant}
We are given a dynamical model of the safety-relevant dynamics in the environment, given as either a discrete-time dynamical system or a system of ordinary differential equations, denoted as \texttt{plant}. We assume that model is correct up to simulation; i.e., if $T(s_i,a,s_j) \not = 0$ for some action $a$, then the dynamical system \textit{plant} maps  $\psi(s_i)$ to a set of states that includes $\psi(s_j)$.
\end{assumption} 

For example, the model may be a system of ODEs that describes how the acceleration and angle impact the future positions of a robot, as well as the potential dynamical behavior of some hazards in the environment. Note that this model only operates on $\mc O$ (the symbolic state space), not $\mc S$ (low-level features such as images or LiDAR).

\begin{assumption}
\label{assm:agent}
We have an abstract model of the agent's behavior, denoted as \texttt{ctrl}, that is correct up to simulation: if $T(s_i,a,s_j) \not = 0$ for some action $a$, then
$\psi(s_j)$ is one of the possible next states after $a(\psi(s_i))$ by \textit{ctrl}.
\end{assumption} 

An abstract model of the agent's behavior describes at a high-level a safe controller behavior, disregarding the fine-grained details an actual controller needs to be efficient. An example is a model that brakes if it is too close to a hazard and can have any other type of behavior otherwise. Note that \texttt{ctrl} is very different from a safe policy $\pi$, since it only models the safety-related aspects of $\pi$ without considering reward optimization.

Assumptions~\ref{assm:eps}-\ref{assm:agent} are mild and reasonable for most practical systems to satisfy.

\section{Background}
\label{sec:background}

\label{subsec:deepRL}
The goal of an RL agent represented as an MDP $(\mc S, \mc A, T, R, \gamma)$  is to find a policy $\pi$ 
that maximizes its expected total reward from an initial state $s_0 \in \sinit$ :
\begin{eqnarray}
V^{\pi}(s) &\triangleq& \mathbb{E}_{\pi}\left[\sum\nolimits_{i=0}^{\infty}\gamma^i r_i)\right] \label{eq:1}
\end{eqnarray}
where $r_i$ is the reward at step $i$.
In a deep RL setting, we can use the DNN parameters $\theta$ to parametrize $\pi(a|s; \theta)$.
One particularly effective implementation and extension of this idea is proximal policy optimization (PPO), which improves sample efficiency and stability by sampling data in batches and then optimizing a surrogate objective function that prevents overly large policy updates \citep{DBLP:journals/corr/SchulmanWDRK17}.
This enables end-to-end learning through gradient descent which significantly reduces the dependency of the learning task on refined domain knowledge.
Deep RL thus provides a key advantage over traditional approaches which were bottle-necked by a manual, time-consuming, and often incomplete feature engineering process.

To ensure formal guarantees we use differential Dynamic Logic (\dL)
\citep{DBLP:journals/jar/Platzer08,Platzer10,DBLP:conf/lics/Platzer12a,DBLP:journals/jar/Platzer17},
a logic for specifying and proving reachability properties of hybrid dynamical systems, which combine both discrete-time (e.g. a robot that decides actions at discrete times) and continuous-time dynamics (e.g. an ODE describing the position of the robot at any time). 
Hybrid systems can be described with hybrid programs (HPs), for which we give an informal definition in \rref{tab:hps}. Notably, besides the familiar program syntax, HPs are able to represent a non-deterministic choice between two programs $\alpha \cup \beta$, and a continuous evolution of a system of ODEs for an arbitrary amount of time, given a domain constraint $F$ on the state space $\{x_1'=\theta_1,...,x_n'=\theta_n \ \& \ F\}$.


Formulas of \dL are generated by the following grammar where $\alpha$ ranges over HPs:
\[
\varphi,\psi ::= f \sim g ~|~  \varphi \land \psi ~|~ \varphi \lor \psi ~|~ \varphi \limply \psi ~|~ \forall x. \varphi ~|~ \exists x. \varphi ~|~ \dibox{\alpha}\varphi
\]
where $f,g$ are polynomials over the state variables, $\phi$ and $\psi$ are formulas of the state variables, $\sim$ is one of $\{ \le, <, =, >, \ge\}$.
The formula $\dibox{\alpha}\varphi$ means that a formula $\varphi$ is true in every state that can be reached by executing the hybrid program $\alpha$.

Given a set of initial conditions \texttt{init} for the initial states, a discrete-time controller \texttt{ctrl} representing the abstract behaviour of the agent, a continuous-time system of ODEs \texttt{plant} representing the environment and a safety property \texttt{safe} we define the \emph{safety preservation problem} 
as verifying that the following holds:
\begin{equation}
\label{eq:safety}
\texttt{init} \limply \dibox{\{\texttt{ctrl};\texttt{plant}\}^*}\texttt{safe}
\end{equation}
Intuitively, this formula means that if the system starts in an intial state that satisfies \texttt{init}, takes one of the (possibly infinite) set of control choices described by \texttt{ctrl}, and then follows the system of ordinary differential equations described by \texttt{plant}, then the system will always remain in states where \texttt{safe} is true.

\begin{example}[Hello, World]
Consider a 1D point-mass $x$ that must avoid colliding with a static obstacle ($o$) and has perception error bounded by $\frac{\epsilon}{2}$. The following \dL model characterizes an infinite set controllers that are all safe, in the sense that $x \not = o$ \emph{for \textbf{all} forward time and at every point throughout the entire flow of the ODE}:
$$
\texttt{init} \limply \dibox{\{\texttt{ctrl}; t:=0; \texttt{plant}\}^*}x-o > \epsilon
$$ 
\begin{align*}
\text{where,} \qquad  \textsf{SB}(a) &\equiv 2B(x-o-\epsilon) > v^2+(a+B)*(aT^2+2Tv)) \\
\texttt{init} &\equiv \textsf{SB}(-B) \land B>0 \land T>0 A > 0 \land v \ge 0 \land \epsilon > 0 \\
\texttt{ctrl} &\equiv a := *; ?-B \le a \le A \land \textsf{SB}(a) \\
\texttt{plant} &\equiv \{ x'=v,v'=a,t'=1 \& t \le T \land v \ge 0 \}
\end{align*}
Starting from any state that satisifies the formula $\texttt{init}$, the (abstract/non-deterministic) controller chooses \textbf{any} acceleration satisfying the $\textsf{SB}$ constraint. 
After choosing any $a$ that satisfies \textsf{SB}, the system then follows the flow of the system of ODEs in \texttt{plant} for any positive amount of time $t$ less than $T$. The constraint $v \ge 0$ simply means that braking (i.e., choosing a negative acceleration) can braing the pointmass to a stop, but cannot cause it to move backwards. 

The full formula says that no matter how many times we execute the controller and then follow the flow of the ODEs, it will always be the case -- again, for an infinite set of permissible controllers -- that $x-o < \epsilon$.
\end{example}

Theorems of \dL can be automatically proven in the KeYmaera~X theorem prover \citep{DBLP:conf/cade/FultonMQVP15,bellerophon}. 
\citep{DBLP:journals/fmsd/MitschP16} explains how to synthesize action space guards from non-deterministic specifications of controllers (\texttt{ctrl}),
and \citet{aaai18,DBLP:conf/tacas/FultonP19} explains how these action space guards are incorporated into reinofrcement learning to esnure safe exploration.
Additional details about how we synthesize monitoring conditions from \dL models is available in \citep{DBLP:journals/fmsd/MitschP16} and in \rref{appendix:monitors}.

\section{\mname: Verifiably Safe RL on Visual Inputs}
\label{sec:fossil}
We present \mname, a framework that can augment any deep RL algorithm to perform \emph{safe exploration} on visual inputs. As discussed in \rref{sec:probdef}, we decompose the general problem in two tasks:
\begin{compactenum}
    \item learning a mapping of visual inputs $s$ into a symbolic state $o$ for safety-relevant properties using only a few examples (described in \rref{sec:symbolicMapping} and shown in \rref{fig:combined}a);
    \item learning policies over visual inputs, while enforcing safety in the symbolic state space (described in \rref{sec:constrainedLearning} and shown in \rref{fig:combined}c).
\end{compactenum}
This latter task requires a controller monitor, which is a function $\varphi : O \times A \rightarrow \{0, 1\}$ that classifies each action $a$ in each symbolic state $o$ as ``safe'' or not. In this paper this monitor is synthesized and verified offline following \citep{aaai18,DBLP:conf/tacas/FultonP19}. In particular, as discussed in the previous sections, the KeYmaera X theorem prover solves the safety preservation problem presented in Eq. \rref{eq:safety} for a set of high-level reward-agnostic safety properties \texttt{safe}, a system of differential equations characterizing the relevant subset of environmental dynamics \texttt{plant}, an abstract description of a safe controller \texttt{ctrl} and a set of initial conditions \texttt{init} (shown in \rref{fig:combined}b).

\begin{figure*}[t]
\centering  
\includegraphics[width=0.9\columnwidth]{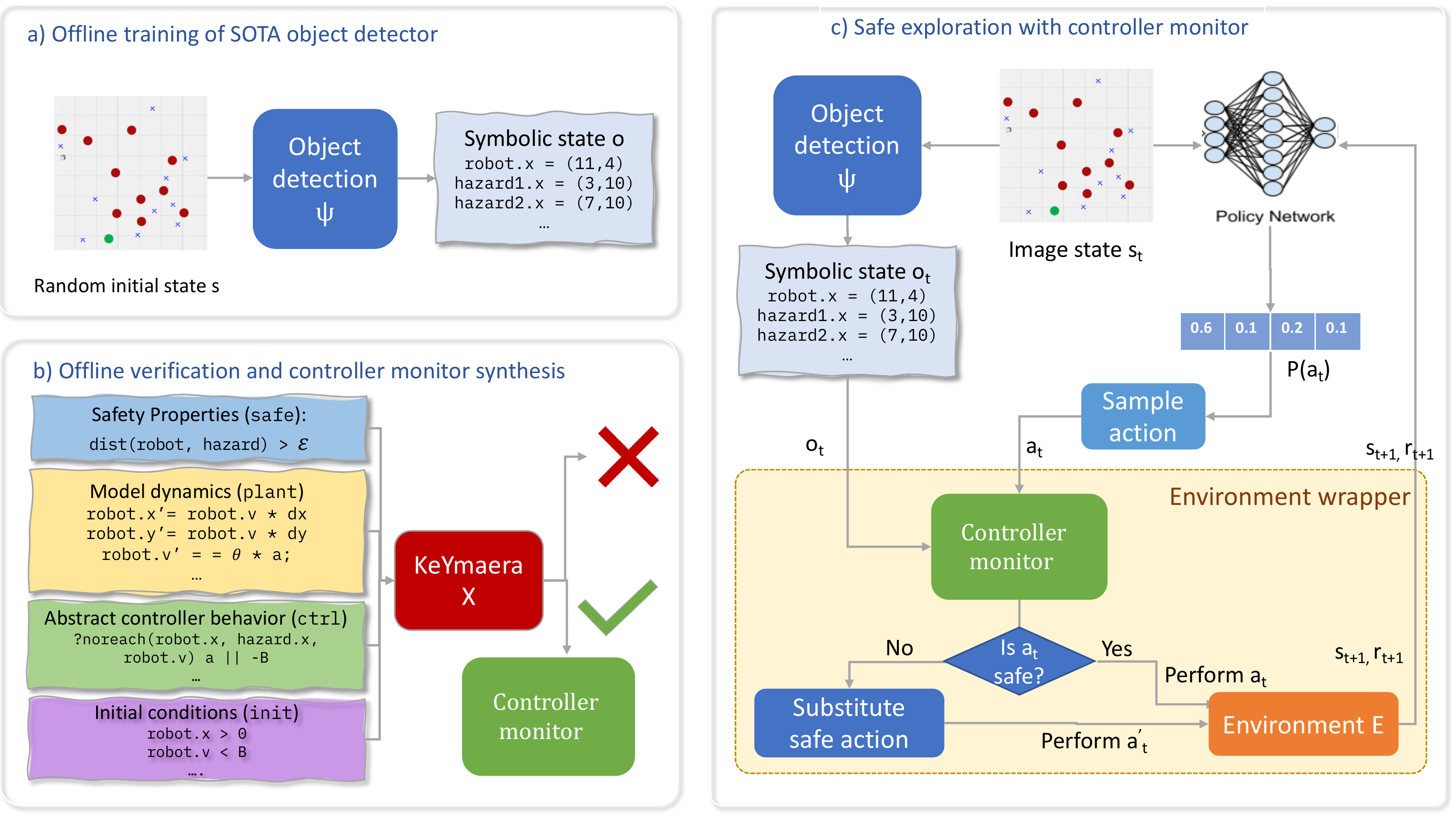}
\caption{\mname\: The left panels a) and b) represent offline pre-processing (described in \rref{sec:symbolicMapping}) and verification. The right panel c) shows how these components are used to safely explore, as described in \rref{sec:constrainedLearning}.}
\label{fig:combined}
\end{figure*}

\subsection{Object Detection}
\label{sec:symbolicMapping}

In order to remove the need to construct labelled datasets for each environment, we only assume that we are given a small set of images of each safety-critical object and a set of background images (in practice, we use 1 image per object and 1 background). We generate synthetic images by pasting the objects onto a background with different locations, rotations, and other augmentations. We then train a CenterNet-style object detector \citep{zhou2019objects_centernet} which performs multi-way classification for whether each pixel is the center of an object. For speed and due to the visual simplicity of the environments, the feature extraction CNN is a truncated ResNet18 \citep{he2016deep_resnet} which only keeps the first residual block. The loss function is the modified focal loss \citep{lin2017focal} from \cite{law2018cornernet}. See \rref{app:symbolicMapping} for full details on the object detector.
Our current implementation does not optimize or dedicate hardware to the object detector, so detection adds some run-time overhead for all environments. However, this is an implementation detail rather than an actual limitation of the approach. There are many existing approaches that make it possible to run object detectors quickly enough for real-time control.

\subsection{Enforcing Constraints}
\label{sec:constrainedLearning}

While \mname can augment any existing deep RL algorithm, this paper extends PPO \citep{DBLP:conf/icml/SchulmanLAJM15}. 
The algorithm performs RL as normal except that, whenever an action is attempted, the object detector and safety monitor are first used to check if the action is safe. If not, a safe action is sampled uniformly at random from the safe actions in the current state. This happens outside of the agent and can be seen as wrapping the environment with a safety check. Pseudocode for performing this wrapping is in \rref{alg:mname}. 
The controller monitor is extracted from a verified \dL model (see Page 3 of \citep{aaai18} for details).
A full code listing that in-lines \rref{alg:mname} into a generic RL algorithm is provided in Appendix \ref{appendix:fullcode}.

\begin{algorithm}[ht]
\caption{The \mname safety guard.} \label{alg:mname}
\begin{algorithmic}
\label{alg:mname}
\STATE \textbf{Input:} $s_t$: input image; $a_t$: input action; $\psi$: object detector; $\varphi$: controller monitor; $E = (\mc S, \mc A, R, T)$: MDP of the original environment \\
\STATE $a'_t = a_t$
\IF {$\lnot \varphi(\psi(s_t), a_t)$}
\STATE Sample substitute safe action $a'_t$ uniformly from $\{ a \in \mc A \mid \varphi(\psi(s_t), a) \}$
\ENDIF
\STATE \textbf{Return} $s_{t+1} \sim T(s_t, a'_t, \cdot)$, $r_{t+1} \sim R(s_t, a'_t)$
\end{algorithmic}
\end{algorithm}

\subsection{Safety and Convergence Results}

We establish two theoretical properties about \mname.
First, we show that \mname safely explores.
Second, we show that if \mname is used on top of an RL algorithm which converges (locally or globally) then \mname will converge to the (locally or globally) optimal safe policy. 
All proofs are in the Appendix.

\begin{theorem}[]\label{thm:safety}
If Assumptions \ref{assm:eps}-\ref{assm:agent} hold along a trajectory $s_0, a_0, s_1, a_1, \dots, a_{n-1}, s_n$ with $s_0 \in \sinit$ for a model of the environment \texttt{plant} and a model of the controller \texttt{ctrl}, where each $a_i$ is chosen based on \rref{alg:mname}, then every state along the trajectory is safe; i.e., $\forall i \ge 0, \texttt{safe}(\psi(s_i))$.
\label{thm:correctness}
\end{theorem}

This results implies that any RL agent augmented with \rref{alg:mname} is always safe during learning. Our second theorem states that any RL agent that is able to learn an optimal policy in an environment $E$  can be combined with \rref{alg:mname} to learn a \emph{reward-optimal safe policy}.

\begin{theorem}\label{thm:policy_equivalence} 
Let $E$ be an environment and $L$ a reinforcement learning algorithm.

If $L$ converges to a reward-optimal policy $\pi^*$ in $E$,
then using \rref{alg:mname} with $L$ converges to $\pi^*_s$, the safe policy with the highest reward (i.e. the \emph{reward-optimal safe policy}).
\end{theorem}


\section{Experimental Validation of \mname}
\label{sec:evaluation}


\begin{figure*}[t]
\centering
\subfigure[\label{fig:xo_illustration}]{\includegraphics[height=0.2\columnwidth]{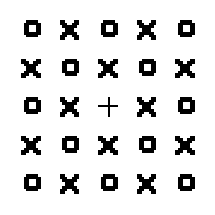}}
\subfigure[\label{fig:acc_illustration}]{\includegraphics[height=0.15\columnwidth]{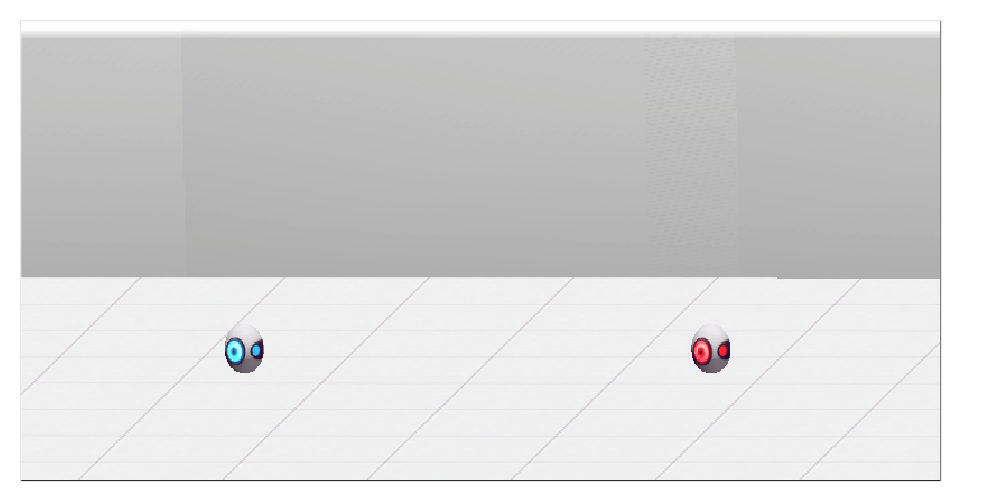}}
\subfigure[\label{fig:gf_illustration}]{\includegraphics[height=0.2\columnwidth]{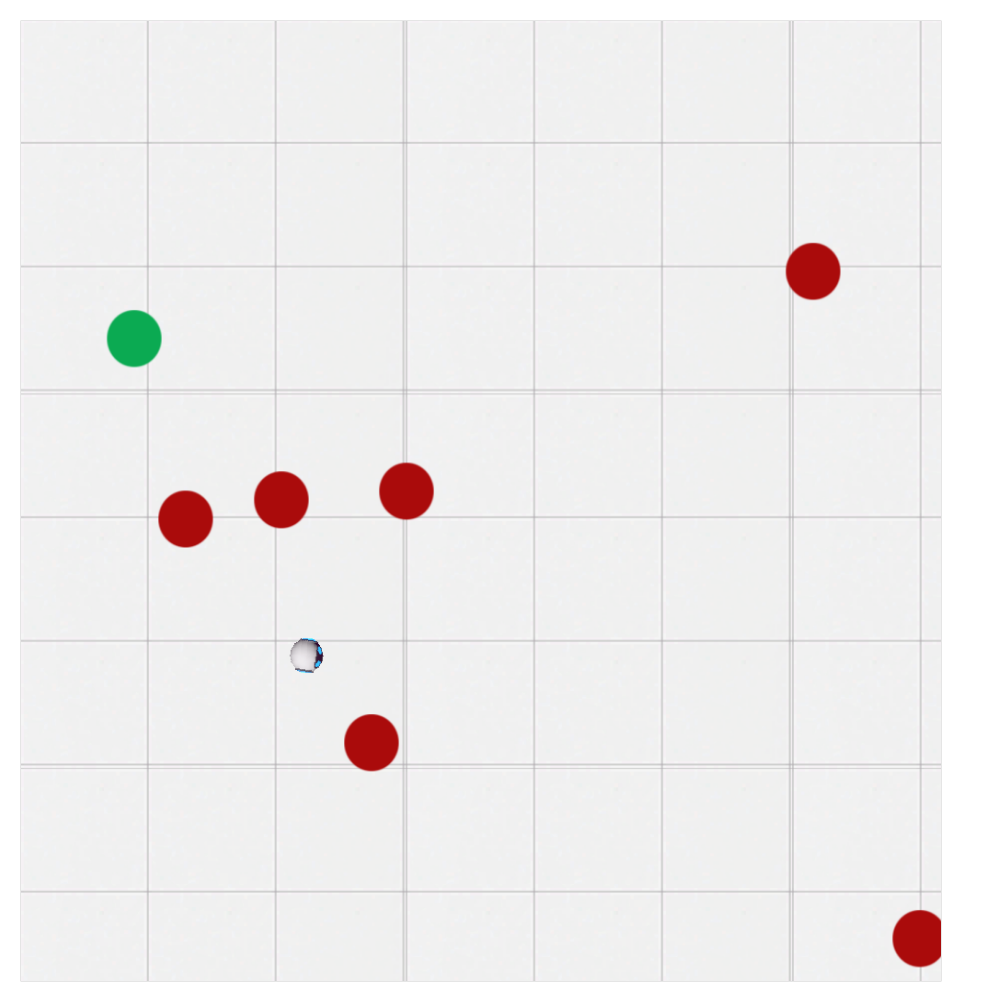}}
\subfigure[\label{fig:pointmesses_illustration}]{\includegraphics[height=0.2\columnwidth]{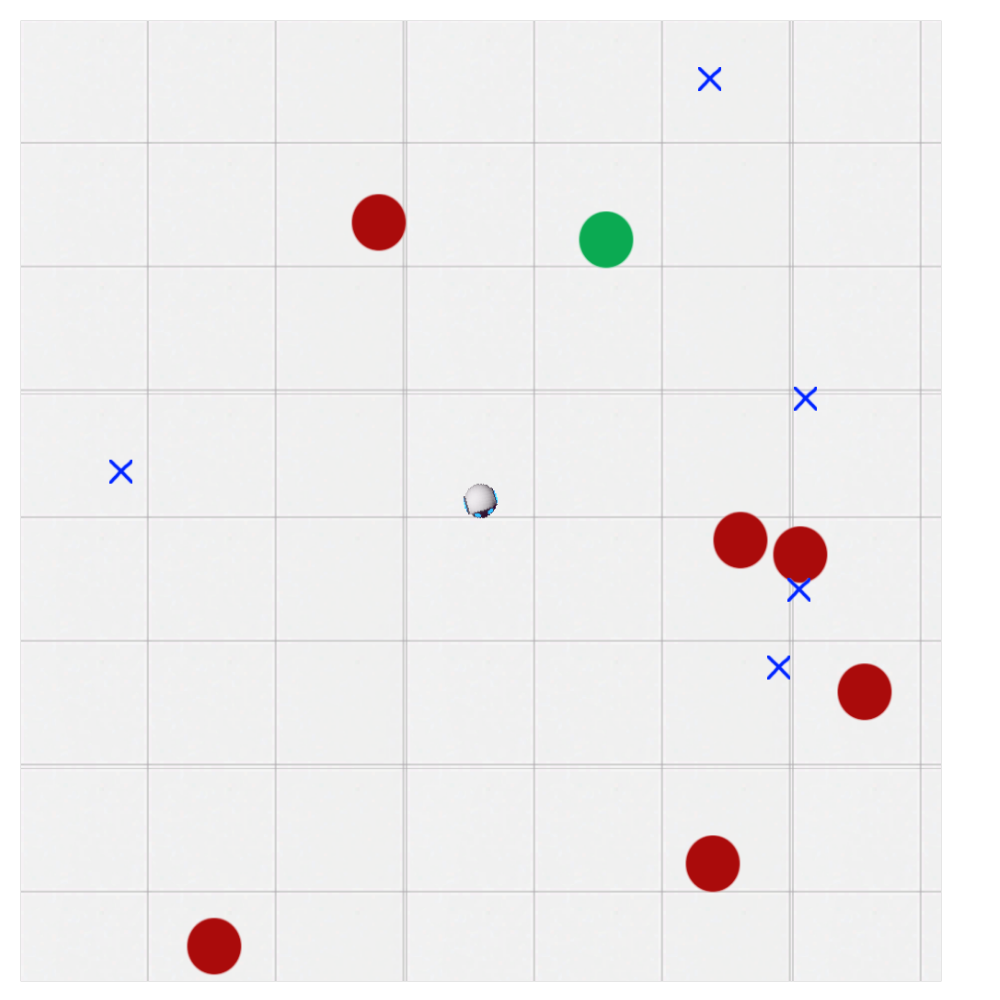}}
\caption{
Visualizations of evaluation environments.
 \textbf{(a)} XO environment
 \textbf{(b)} ACC environment
 \textbf{(c)} Goal-finding environment
 \textbf{(d)} Pointmess environment
}
\end{figure*}

We evaluate \mname on four environments:
a discrete \textbf{XO} environment \citep{garnelo2016towards},
an adaptive cruise control environment (ACC),
a 2D goal-finding environment similar to the Open AI Safety Gym Goal environment \citep{Ray2019} but without a MuJoCo dependency (GF), and
a pointmesses environment that emphasizes the problem of preventing reward hacking in safe exploration systems (PM).
\mname explores each environment without encountering any unsafe states.

\begin{figure}[ht]
\centering
\subfigure[\label{fig:xo_safe}]{\includegraphics[width=0.24\columnwidth]{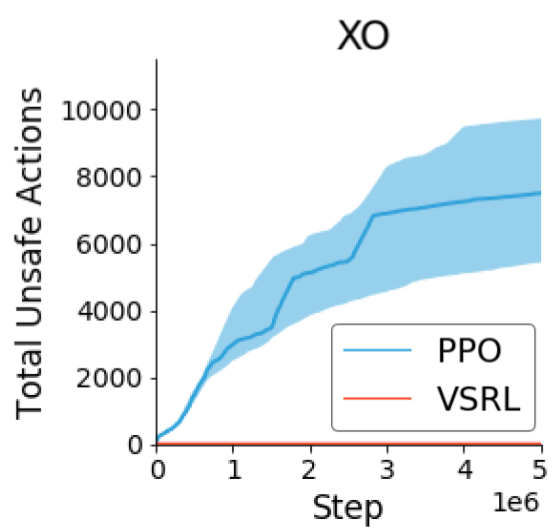}}
\subfigure[\label{fig:acc_safe}]{\includegraphics[width=0.24\columnwidth]{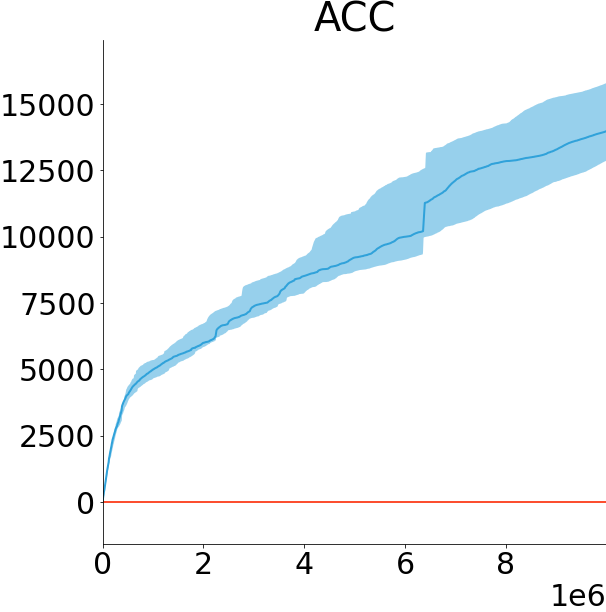}}
\subfigure[\label{fig:pmgf_safe}]{\includegraphics[width=0.24\columnwidth]{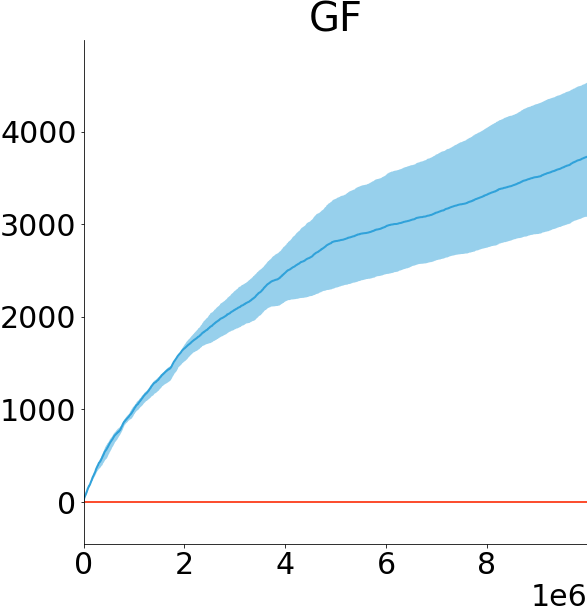}}
\subfigure[\label{fig:pm_safe}]{\includegraphics[width=0.24\columnwidth]{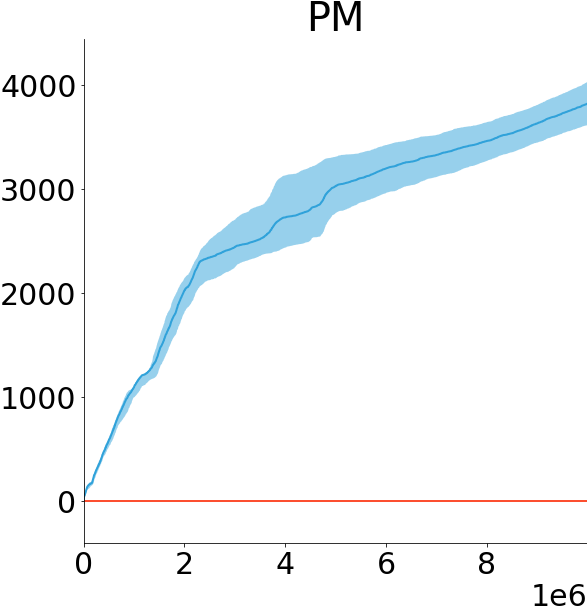}}

\subfigure[\label{fig:xo_reward}]{\includegraphics[width=0.24\columnwidth]{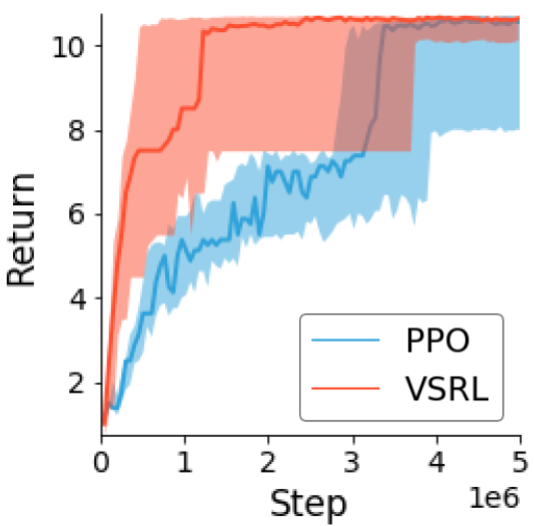}}
\subfigure[\label{fig:acc_reward}]{\includegraphics[width=0.24\columnwidth]{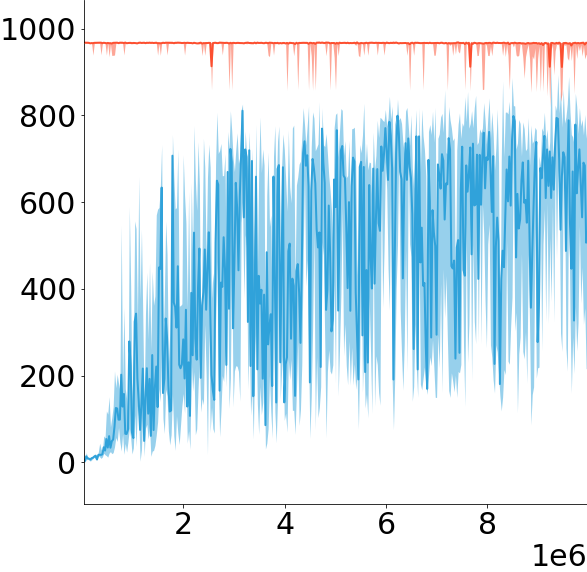}}
\subfigure[\label{fig:pmgf_reward}]{\includegraphics[width=0.24\columnwidth]{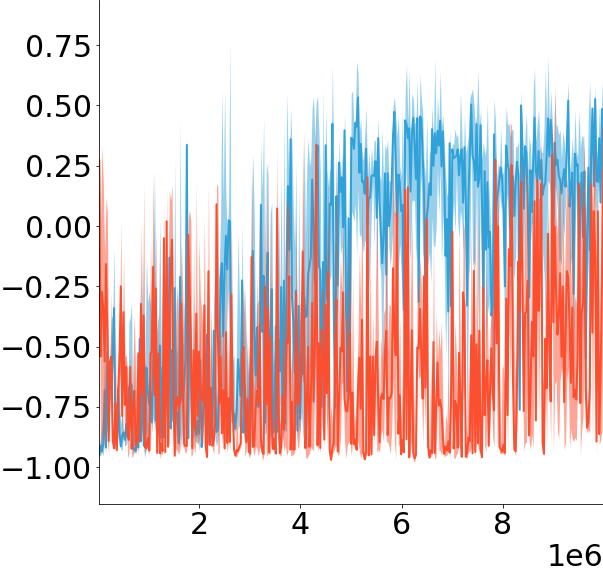}}
\subfigure[\label{fig:pm_reward}]{\includegraphics[width=0.24\columnwidth]{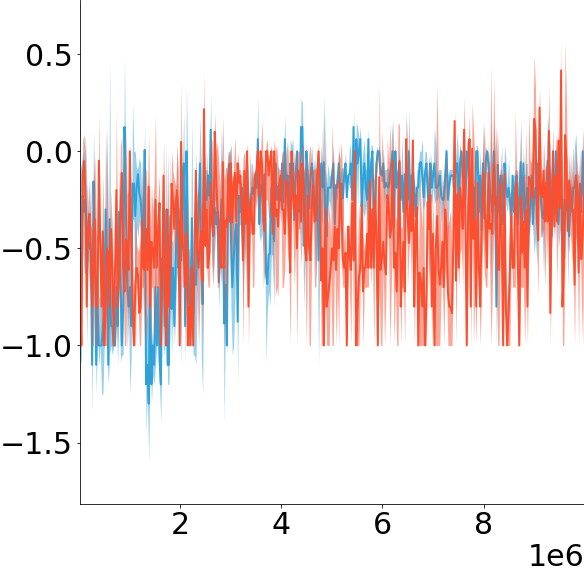}}

\caption{
{\small
Empirical evaluation of \mname on all environments. All plots show the median and interquartile range of 4+ repeats.
}}
\label{fig:reward_and_safety}
\end{figure}

The \textbf{XO} Environment is a simple setting introduced by \citep{garnelo2016towards} for demonstrating symbolic reinforcement learning algorithms (the implementation by \citet{garnelo2016towards} was unavailable, so we reimplemented this environment). 
The \textbf{XO} environment, visualized in \rref{fig:xo_illustration}, contains three types of objects: \textbf{X} objects that must be collected (+1 reward), \textbf{O} objects that must be avoided (-1 reward), and the agent (marked by a \textbf{+}). 
There is also a small penalty (-0.01) at each step to encourage rapid collection of all \textbf{X}s and completion of the episode. 
This environment provides a simple baseline for evaluating \mname. 
It is also simple to modify and extend, which we use to evaluate the ability of \mname to generalize safe policies to environments that deviate slightly from implicit modeling assumptions. 
The symbolic state space includes the position of the \textbf{+} and the \textbf{O}, but not the position of the \textbf{X}s because they are not safety-relevant. The purpose of this benchmark is to provide a benchmark for safe exploration in a simple discrete setting.

The adaptive cruise control (ACC) environment has two objects: a follower and a leader. The follower must maintain a fixed distance from the leader without either running into the leader or following too far behind.  We use the verified model from \citep{DBLP:journals/sttt/QueselMLAP16} to constrain the agent's dynamics.

The 2D goal-finding environment consists of an agent, a set of obstacles, and a goal state. The obstacles are the red circles and the goal state is the green circle. The agent must navigate from its (random) starting position to the goal state without encountering any of the obstacles. Unlike the OpenAI Safety Gym, the obstacles are \emph{hard} safety constraints; i.e., the episode ends if the agent hits a hazard. We use the verified model from \citep{DBLP:conf/rss/MitschGP13} to constrain the agent's dynamics.

The 2D pointmesses environment consists of an agent, a set of obstacles, a goal state, and a set of pointmesses (blue Xs). The agent receives reward for picking up the pointmesses, and the episode ends when the agent picks up all messes and reaches the goal state. Unlike the 2D goal-finding environment, hitting an obstacle does not end the episode. Instead, the obstacle is removed from the environment and a random number of new pointmesses spawn in its place. Notice that this means that the agent may reward hack by taking an unsafe action (hitting an obstacle) and then cleaning up the resulting pointmesses. We consider this the incorrect behavior. We use the verified model from \citep{DBLP:conf/rss/MitschGP13} to constrain the agent's dynamics.

\begin{table}[t]
  \centering
  \begin{tabular}{lcccccccc}
    \multicolumn{1}{l}{} & \multicolumn{2}{c}{XO} & \multicolumn{2}{c}{ACC} & \multicolumn{2}{c}{GF} & \multicolumn{2}{c}{PM}\\
    \cmidrule(lr){2-3}\cmidrule(lr){4-5}\cmidrule(lr){6-7}\cmidrule(lr){8-9}
    Method & R & U & R & U & R & U & R & U\\
    \midrule
    PPO & 10.5 & 7500 & 529 & 13983 & 0.233 & 3733 & -0.25 & 3819\\
    VSRL & 10.5 & 0 & 967 & 0 & 0.228 & 0 & -0.225 & 0 \\
    \bottomrule
  \end{tabular}
  \vspace{1em}
  \caption{Final reward (R; higher is better) and total number of unsafe actions (U; lower is better) on all environments. All results are the median over at least 4 replicates.} \label{tab:results}
\end{table}

We compare \mname to PPO using two metrics: the number of safety violations during training and the cumulative reward. These results are
summarized in \rref{tab:results}.
\mname is able to perfectly preserve safety in all environments from the beginning of training even with the $\epsilon$-bounded errors in extracting the symbolic features from the images. In contrast, vanilla PPO takes many unsafe actions while training and does not always converge to a policy that entirely avoids unsafe objects by the end of training.

In some environments, preserving safety specifications also substantially improves sample efficiency and policy performance early in the training process. In the ACC environment, in particular, it is very easy to learn a safe policy which is reward-optimal. In the GF and PM environments, both the baseline agent and the \mname agent struggle to learn to perform the task well (note that these tasks are quite difficult because encountering an obstacle ends the episode). However, \mname remains safe without losing much reward relative to the amount of uncertainty in both policies. See \rref{appendix:fullcode} for details on our experimental evaluation and implementation.

\section{Related Work}
\label{sec: relwork}
Recently, there has been a growing interest in safe RL, especially in the context of \emph{safe exploration}, where the agent has to be safe also during training.
A naive approach to RL safety is reward shaping, in which one defines a penalty cost for unsafe actions. This approach has several drawbacks, e.g. the choice of the penalty is brittle, so a naive choice may not outweight a shorter path to the reward, as shown by \citet{dalal2018safe}.
Therefore, recent work on safe RL addresses the challenge of providing reward-agnostic safety guarantees for deep RL \citep{garcia2015comprehensive,xiang2018verification}.
Many recent safe exploration methods focus on safety guarantees that hold in expectation (e.g., \citep{DBLP:conf/icml/SchulmanLAJM15,DBLP:conf/icml/AchiamHTA17}) or with high probability (e.g., \citep{berkenkamp2017safe,dalal2018safe,koller2018learning,cheng2019end}. Some of these approaches achieve impressive results by drawing upon techniques from control theory, such as Lyapunov functions \citep{berkenkamp2017safe} and control barrier certificates.

On the other hand, ensuring safety in expectation or with high probability is generally not sufficient in safety-critical settings where guarantees must hold \emph{always}, even for rare and measure-zero events. 
Numerical testing alone cannot provide such guarantees in practice \citep{RANDDriveToSafety} or even in theory \citep{DBLP:conf/hybrid/PlatzerC07}. 
The problem of providing formal guarantees in RL is called \emph{formally constrained reinforcement learning (FCRL)}. 
Existing FCRL methods such as \citep{HasanbeigKroening,hasanbeig2018,Hasanbeig2019, hasanbeig2020cautious, DBLP:conf/tacas/HahnPSSTW19,DBLP:conf/aaai/AlshiekhBEKNT18,aaai18,DBLP:journals/corr/neuralsimplex,de2019foundations} combine the best of both worlds: they optimize for a reward function while still providing formal safety guarantees. 
While most FCRL method can only ensure the safety in discrete-time environments known a priori,
\citet{aaai18,DBLP:conf/tacas/FultonP19} introduce \emph{Justified Speculative Control}, which exploits Differential Dynamic Logic\citep{DBLP:conf/cade/Platzer15} to prove the safety of \emph{hybrid systems}, systems that combine an agent's discrete-time decisions with a continuous time dynamics of the system. 

A major drawback of current FCRL methods 
is that they only learn control policies over handcrafted symbolic state spaces. 
While many methods extract a symbolic mapping for RL from visual data, e.g. \citep{lyu2019sdrl,yang2018peorl,yang2019program,lu2018robot,garnelo2016towards,li2018object_odrl,liang2018task,goel2018unsupervised_morel}, they all require that all of the reward-relevant features are explicitly represented in the symbolic space. As shown by the many successes of Deep RL, e.g. \citep{mnih-atari-2013}, handcrafted features often miss important signals hidden in the raw data.

Our approach aims at combining the best of FCRL and end-to-end RL to ensure that exploration is always safe with formal guarantees, while allowing a deep RL algorithm to fully exploit the visual inputs for reward optimization.


\section{Conclusion and Discussions}
\label{sec:conclusion}

Safe exploration in the presence of hard safety constraints is a schallenging problem in reinforcement learning. We contribute \mname, an approach toward safe learning on visual inputs. Through theoretical analysis and experimental evaluation, this paper establishes that \mname maintains perfect safety during exploration while obtaining comparable reward.
Because \mname separates safety-critical object detection from RL, next steps should include applying tools from adversarial robustness to the object detectors used by \mname.

\bibliographystyle{icml2020}
\bibliography{main}

\begin{thebibliography}{51}
\providecommand{\natexlab}[1]{#1}
\providecommand{\url}[1]{\texttt{#1}}
\expandafter\ifx\csname urlstyle\endcsname\relax
  \providecommand{\doi}[1]{doi: #1}\else
  \providecommand{\doi}{doi: \begingroup \urlstyle{rm}\Url}\fi

\bibitem[Achiam et~al.(2017)Achiam, Held, Tamar, and
  Abbeel]{DBLP:conf/icml/AchiamHTA17}
Achiam, J., Held, D., Tamar, A., and Abbeel, P.
\newblock Constrained policy optimization.
\newblock In Precup, D. and Teh, Y.~W. (eds.), \emph{International Conference
  on Machine Learning ({ICML} 2017)}, volume~70 of \emph{Proceedings of Machine
  Learning Research}, pp.\  22--31. {PMLR}, 2017.

\bibitem[Alshiekh et~al.(2018)Alshiekh, Bloem, Ehlers, K{\"{o}}nighofer,
  Niekum, and Topcu]{DBLP:conf/aaai/AlshiekhBEKNT18}
Alshiekh, M., Bloem, R., Ehlers, R., K{\"{o}}nighofer, B., Niekum, S., and
  Topcu, U.
\newblock Safe reinforcement learning via shielding.
\newblock In \emph{{AAAI} Conference on Artificial Intelligence}, 2018.

\bibitem[Berkenkamp et~al.(2017)Berkenkamp, Turchetta, Schoellig, and
  Krause]{berkenkamp2017safe}
Berkenkamp, F., Turchetta, M., Schoellig, A., and Krause, A.
\newblock Safe model-based reinforcement learning with stability guarantees.
\newblock In \emph{Advances in neural information processing systems}, pp.\
  908--918, 2017.

\bibitem[Cheng et~al.(2019)Cheng, Orosz, Murray, and Burdick]{cheng2019end}
Cheng, R., Orosz, G., Murray, R.~M., and Burdick, J.~W.
\newblock End-to-end safe reinforcement learning through barrier functions for
  safety-critical continuous control tasks.
\newblock In \emph{Proceedings of the AAAI Conference on Artificial
  Intelligence}, volume~33, pp.\  3387--3395, 2019.

\bibitem[Clarke et~al.(2018)Clarke, Henzinger, Veith, and
  Bloem]{clarke2018handbook}
Clarke, E.~M., Henzinger, T.~A., Veith, H., and Bloem, R. (eds.).
\newblock \emph{Handbook of Model Checking}.
\newblock Springer, 2018.

\bibitem[Dalal et~al.(2018)Dalal, Dvijotham, Vecerik, Hester, Paduraru, and
  Tassa]{dalal2018safe}
Dalal, G., Dvijotham, K., Vecerik, M., Hester, T., Paduraru, C., and Tassa, Y.
\newblock Safe exploration in continuous action spaces.
\newblock \emph{arXiv preprint arXiv:1801.08757}, 2018.

\bibitem[De~Giacomo et~al.(2019)De~Giacomo, Iocchi, Favorito, and
  Patrizi]{de2019foundations}
De~Giacomo, G., Iocchi, L., Favorito, M., and Patrizi, F.
\newblock Foundations for restraining bolts: Reinforcement learning with
  ltlf/ldlf restraining specifications.
\newblock In \emph{International Conference on Automated Planning and
  Scheduling ({ICAPS} 2019)}, 2019.

\bibitem[Espeholt et~al.(2018)Espeholt, Soyer, Munos, Simonyan, Mnih, Ward,
  Doron, Firoiu, Harley, Dunning, et~al.]{espeholt2018impala}
Espeholt, L., Soyer, H., Munos, R., Simonyan, K., Mnih, V., Ward, T., Doron,
  Y., Firoiu, V., Harley, T., Dunning, I., et~al.
\newblock Impala: Scalable distributed deep-rl with importance weighted
  actor-learner architectures.
\newblock \emph{arXiv preprint arXiv:1802.01561}, 2018.

\bibitem[Fulton \& Platzer(2018)Fulton and Platzer]{aaai18}
Fulton, N. and Platzer, A.
\newblock Safe reinforcement learning via formal methods: Toward safe control
  through proof and learning.
\newblock In \emph{{AAAI} Conference on Artificial Intelligence}, 2018.

\bibitem[Fulton \& Platzer(2019)Fulton and Platzer]{DBLP:conf/tacas/FultonP19}
Fulton, N. and Platzer, A.
\newblock Verifiably safe off-model reinforcement learning.
\newblock In Vojnar, T. and Zhang, L. (eds.), \emph{{TACAS} 2019}, volume 11427
  of \emph{Lecture Notes in Computer Science}, pp.\  413--430. Springer, 2019.
\newblock ISBN 978-3-030-17461-3.
\newblock \doi{10.1007/978-3-030-17462-0\_28}.

\bibitem[Fulton et~al.(2015)Fulton, Mitsch, Quesel, V{\"o}lp, and
  Platzer]{DBLP:conf/cade/FultonMQVP15}
Fulton, N., Mitsch, S., Quesel, J.-D., V{\"o}lp, M., and Platzer, A.
\newblock {KeYmaera X}: An axiomatic tactical theorem prover for hybrid
  systems.
\newblock In \emph{CADE}, 2015.

\bibitem[Fulton et~al.(2017)Fulton, Mitsch, Bohrer, and Platzer]{bellerophon}
Fulton, N., Mitsch, S., Bohrer, B., and Platzer, A.
\newblock Bellerophon: Tactical theorem proving for hybrid systems.
\newblock In \emph{International Conference on Interactive Theorem Proving},
  2017.

\bibitem[Garc{\i}a \& Fern{\'a}ndez(2015)Garc{\i}a and
  Fern{\'a}ndez]{garcia2015comprehensive}
Garc{\i}a, J. and Fern{\'a}ndez, F.
\newblock A comprehensive survey on safe reinforcement learning.
\newblock \emph{Journal of Machine Learning Research}, 2015.

\bibitem[Garnelo et~al.(2016)Garnelo, Arulkumaran, and
  Shanahan]{garnelo2016towards}
Garnelo, M., Arulkumaran, K., and Shanahan, M.
\newblock Towards deep symbolic reinforcement learning.
\newblock \emph{arXiv preprint arXiv:1609.05518}, 2016.

\bibitem[Goel et~al.(2018)Goel, Weng, and Poupart]{goel2018unsupervised_morel}
Goel, V., Weng, J., and Poupart, P.
\newblock Unsupervised video object segmentation for deep reinforcement
  learning.
\newblock In \emph{Advances in Neural Information Processing Systems}, 2018.

\bibitem[Hahn et~al.(2019)Hahn, Perez, Schewe, Somenzi, Trivedi, and
  Wojtczak]{DBLP:conf/tacas/HahnPSSTW19}
Hahn, E.~M., Perez, M., Schewe, S., Somenzi, F., Trivedi, A., and Wojtczak, D.
\newblock Omega-regular objectives in model-free reinforcement learning.
\newblock In \emph{{TACAS} 2019}, 2019.

\bibitem[Hasanbeig et~al.(2018{\natexlab{a}})Hasanbeig, Abate, and
  Kroening]{HasanbeigKroening}
Hasanbeig, M., Abate, A., and Kroening, D.
\newblock Logically-correct reinforcement learning.
\newblock \emph{CoRR}, abs/1801.08099, 2018{\natexlab{a}}.

\bibitem[Hasanbeig et~al.(2018{\natexlab{b}})Hasanbeig, Abate, and
  Kroening]{hasanbeig2018}
Hasanbeig, M., Abate, A., and Kroening, D.
\newblock Logically-constrained reinforcement learning.
\newblock \emph{arXiv preprint arXiv:1801.08099}, 2018{\natexlab{b}}.

\bibitem[{Hasanbeig} et~al.(2019){Hasanbeig}, {Kantaros}, {Abate}, {Kroening},
  {Pappas}, and {Lee}]{Hasanbeig2019}
{Hasanbeig}, M., {Kantaros}, Y., {Abate}, A.~r., {Kroening}, D., {Pappas},
  G.~J., and {Lee}, I.
\newblock {Reinforcement Learning for Temporal Logic Control Synthesis with
  Probabilistic Satisfaction Guarantees}.
\newblock \emph{arXiv e-prints}, art. arXiv:1909.05304, September 2019.

\bibitem[Hasanbeig et~al.(2020)Hasanbeig, Abate, and
  Kroening]{hasanbeig2020cautious}
Hasanbeig, M., Abate, A., and Kroening, D.
\newblock Cautious reinforcement learning with logical constraints.
\newblock \emph{arXiv preprint arXiv:2002.12156}, 2020.

\bibitem[He et~al.(2016)He, Zhang, Ren, and Sun]{he2016deep_resnet}
He, K., Zhang, X., Ren, S., and Sun, J.
\newblock Deep residual learning for image recognition.
\newblock In \emph{Proceedings of the IEEE conference on computer vision and
  pattern recognition}, pp.\  770--778, 2016.

\bibitem[{ISO-26262}(2011)]{iso26262}
{ISO-26262}.
\newblock {International Organization for Standardization} 26262 road vehicles
  – functional safety.
\newblock 2011.

\bibitem[Kalra \& Paddock(2016)Kalra and Paddock]{RANDDriveToSafety}
Kalra, N. and Paddock, S.~M.
\newblock \emph{Driving to Safety: How Many Miles of Driving Would It Take to
  Demonstrate Autonomous Vehicle Reliability?}
\newblock RAND Corporation, 2016.

\bibitem[Kingma \& Ba(2014)Kingma and Ba]{kingma2014adam}
Kingma, D.~P. and Ba, J.
\newblock Adam: A method for stochastic optimization.
\newblock \emph{arXiv preprint arXiv:1412.6980}, 2014.

\bibitem[Koller et~al.(2018)Koller, Berkenkamp, Turchetta, and
  Krause]{koller2018learning}
Koller, T., Berkenkamp, F., Turchetta, M., and Krause, A.
\newblock Learning-based model predictive control for safe exploration.
\newblock In \emph{2018 IEEE Conference on Decision and Control (CDC)}, pp.\
  6059--6066. IEEE, 2018.

\bibitem[Law \& Deng(2018)Law and Deng]{law2018cornernet}
Law, H. and Deng, J.
\newblock Cornernet: Detecting objects as paired keypoints.
\newblock In \emph{European Conference on Computer Vision}, 2018.

\bibitem[Li et~al.(2018)Li, Sycara, and Iyer]{li2018object_odrl}
Li, Y., Sycara, K., and Iyer, R.
\newblock Object-sensitive deep reinforcement learning.
\newblock \emph{arXiv preprint arXiv:1809.06064}, 2018.

\bibitem[Liang \& Boularias(2018)Liang and Boularias]{liang2018task}
Liang, J. and Boularias, A.
\newblock Task-relevant object discovery and categorization for playing
  first-person shooter games.
\newblock \emph{arXiv preprint arXiv:1806.06392}, 2018.

\bibitem[Lin et~al.(2017)Lin, Goyal, Girshick, He, and
  Doll{\'a}r]{lin2017focal}
Lin, T.-Y., Goyal, P., Girshick, R., He, K., and Doll{\'a}r, P.
\newblock Focal loss for dense object detection.
\newblock In \emph{IEEE international conference on computer vision}, 2017.

\bibitem[Lu et~al.(2018)Lu, Zhang, Stone, and Chen]{lu2018robot}
Lu, K., Zhang, S., Stone, P., and Chen, X.
\newblock Robot representing and reasoning with knowledge from reinforcement
  learning.
\newblock \emph{arXiv preprint arXiv:1809.11074}, 2018.

\bibitem[Lyu et~al.(2019)Lyu, Yang, Liu, and Gustafson]{lyu2019sdrl}
Lyu, D., Yang, F., Liu, B., and Gustafson, S.
\newblock {SDRL}: interpretable and data-efficient deep reinforcement learning
  leveraging symbolic planning.
\newblock In \emph{AAAI'19}, 2019.

\bibitem[Mitsch \& Platzer(2016)Mitsch and
  Platzer]{DBLP:journals/fmsd/MitschP16}
Mitsch, S. and Platzer, A.
\newblock {ModelPlex}: Verified runtime validation of verified cyber-physical
  system models.
\newblock \emph{Form. Methods Syst. Des.}, 49\penalty0 (1):\penalty0 33--74,
  2016.
\newblock Special issue of selected papers from RV'14.

\bibitem[Mitsch et~al.(2013)Mitsch, Ghorbal, and
  Platzer]{DBLP:conf/rss/MitschGP13}
Mitsch, S., Ghorbal, K., and Platzer, A.
\newblock On provably safe obstacle avoidance for autonomous robotic ground
  vehicles.
\newblock In Newman, P., Fox, D., and Hsu, D. (eds.), \emph{Robotics: Science
  and Systems}, 2013.

\bibitem[Mnih et~al.(2013)Mnih, Kavukcuoglu, Silver, Graves, Antonoglou,
  Wierstra, and Riedmiller]{mnih-atari-2013}
Mnih, V., Kavukcuoglu, K., Silver, D., Graves, A., Antonoglou, I., Wierstra,
  D., and Riedmiller, M.
\newblock Playing atari with deep reinforcement learning.
\newblock In \emph{NIPS Deep Learning Workshop}. 2013.

\bibitem[Phan et~al.(2019)Phan, Paoletti, Grosu, Jansen, Smolka, and
  Stoller]{DBLP:journals/corr/neuralsimplex}
Phan, D., Paoletti, N., Grosu, R., Jansen, N., Smolka, S.~A., and Stoller,
  S.~D.
\newblock Neural simplex architecture.
\newblock 2019.

\bibitem[Platzer(2008)]{DBLP:journals/jar/Platzer08}
Platzer, A.
\newblock Differential dynamic logic for hybrid systems.
\newblock \emph{J. Autom. Reas.}, 41\penalty0 (2):\penalty0 143--189, 2008.

\bibitem[Platzer(2010)]{Platzer10}
Platzer, A.
\newblock \emph{Logical Analysis of Hybrid Systems: Proving Theorems for
  Complex Dynamics}.
\newblock Springer, Heidelberg, 2010.

\bibitem[Platzer(2012)]{DBLP:conf/lics/Platzer12a}
Platzer, A.
\newblock Logics of dynamical systems.
\newblock In \emph{LICS}, pp.\  13--24. IEEE, 2012.

\bibitem[Platzer(2015)]{DBLP:conf/cade/Platzer15}
Platzer, A.
\newblock A uniform substitution calculus for differential dynamic logic.
\newblock In \emph{CADE}, 2015.

\bibitem[Platzer(2017)]{DBLP:journals/jar/Platzer17}
Platzer, A.
\newblock A complete uniform substitution calculus for differential dynamic
  logic.
\newblock \emph{J. Autom. Reas.}, 59\penalty0 (2):\penalty0 219--266, 2017.

\bibitem[Platzer \& Clarke(2007)Platzer and
  Clarke]{DBLP:conf/hybrid/PlatzerC07}
Platzer, A. and Clarke, E.~M.
\newblock The image computation problem in hybrid systems model checking.
\newblock In Bemporad, A., Bicchi, A., and Buttazzo, G. (eds.), \emph{HSCC},
  volume 4416 of \emph{LNCS}, pp.\  473--486. Springer, 2007.
\newblock ISBN 978-3-540-71492-7.
\newblock \doi{10.1007/978-3-540-71493-4_37}.

\bibitem[Quesel et~al.(2016)Quesel, Mitsch, Loos, Arechiga, and
  Platzer]{DBLP:journals/sttt/QueselMLAP16}
Quesel, J., Mitsch, S., Loos, S.~M., Arechiga, N., and Platzer, A.
\newblock How to model and prove hybrid systems with {KeYmaera}: a tutorial on
  safety.
\newblock \emph{{STTT}}, 18\penalty0 (1):\penalty0 67--91, 2016.

\bibitem[Ray et~al.(2019)Ray, Achiam, and Amodei]{Ray2019}
Ray, A., Achiam, J., and Amodei, D.
\newblock {Benchmarking Safe Exploration in Deep Reinforcement Learning}.
\newblock 2019.

\bibitem[Schulman et~al.(2015)Schulman, Levine, Abbeel, Jordan, and
  Moritz]{DBLP:conf/icml/SchulmanLAJM15}
Schulman, J., Levine, S., Abbeel, P., Jordan, M.~I., and Moritz, P.
\newblock Trust region policy optimization.
\newblock In Bach, F.~R. and Blei, D.~M. (eds.), \emph{Proceedings of the 32nd
  International Conference on Machine Learning ({ICML} 2015)}, volume~37 of
  \emph{{JMLR} Workshop and Conference Proceedings}, pp.\  1889--1897, 2015.

\bibitem[Schulman et~al.(2017)Schulman, Wolski, Dhariwal, Radford, and
  Klimov]{DBLP:journals/corr/SchulmanWDRK17}
Schulman, J., Wolski, F., Dhariwal, P., Radford, A., and Klimov, O.
\newblock Proximal policy optimization algorithms.
\newblock 2017.
\newblock URL \url{http://arxiv.org/abs/1707.06347}.

\bibitem[Stooke \& Abbeel(2019)Stooke and Abbeel]{stooke2019rlpyt}
Stooke, A. and Abbeel, P.
\newblock rlpyt: A research code base for deep reinforcement learning in
  pytorch, 2019.

\bibitem[Sutton \& Barto(1998)Sutton and Barto]{sutton.barto:reinforcement}
Sutton, R.~S. and Barto, A.~G.
\newblock \emph{Reinforcement Learning: An Introduction}.
\newblock MIT Press, Cambridge, MA, 1998.

\bibitem[Xiang et~al.(2018)Xiang, Musau, Wild, Lopez, Hamilton, Yang,
  Rosenfeld, and Johnson]{xiang2018verification}
Xiang, W., Musau, P., Wild, A.~A., Lopez, D.~M., Hamilton, N., Yang, X.,
  Rosenfeld, J., and Johnson, T.~T.
\newblock Verification for machine learning, autonomy, and neural networks
  survey.
\newblock \emph{arXiv}, 2018.

\bibitem[Yang et~al.(2018)Yang, Lyu, Liu, and Gustafson]{yang2018peorl}
Yang, F., Lyu, D., Liu, B., and Gustafson, S.
\newblock Peorl: Integrating symbolic planning and hierarchical reinforcement
  learning for robust decision-making.
\newblock \emph{arXiv preprint arXiv:1804.07779}, 2018.

\bibitem[Yang et~al.(2019)Yang, Gustafson, Elkholy, Lyu, and
  Liu]{yang2019program}
Yang, F., Gustafson, S., Elkholy, A., Lyu, D., and Liu, B.
\newblock Program search for machine learning pipelines leveraging symbolic
  planning and reinforcement learning.
\newblock In \emph{Genetic Programming Theory and Practice XVI}. 2019.

\bibitem[Zhou et~al.(2019)Zhou, Wang, and
  Kr{\"a}henb{\"u}hl]{zhou2019objects_centernet}
Zhou, X., Wang, D., and Kr{\"a}henb{\"u}hl, P.
\newblock Objects as points.
\newblock \emph{arXiv preprint arXiv:1904.07850}, 2019.

\end{thebibliography}


\clearpage
\section*{Supplementary material for: Verifiably Safe Exploration for End-to-End
Reinforcement Learning}
\appendix

\section{Model Monitoring}\label{appendix:monitors}

We use differential Dynamic Logic (\dL) \citep{DBLP:journals/jar/Platzer08,Platzer10,DBLP:conf/lics/Platzer12a,DBLP:conf/cade/Platzer15,DBLP:journals/jar/Platzer17} to specify safety constraints on the agent's action space. 
\dL is a logic for specifying and proving reachability properties of both discrete and continuous time dynamical systems. 

In this section we expand on the definitions and
provide some illustrative examples.
In particular, we focus on the language of hybrid programs (HPs), their reachability logic (\dL), and monitor synthesis for \dL formulas.

\subsection{Hybrid Programs Overview}

As shown succinctly in \rref{tab:hps}, hybrid programs are a simple programming language that combines imperative programs with systems of differential equations. 
We expand the description from \rref{tab:hps} and define the syntax and informal semantics of HPs are as follows:
\begin{itemize}
\item $\alpha;\beta$ executes $\alpha$ and then executes $\beta$. 
\item $\alpha \cup \beta$ executes either $\alpha$ or $\beta$ nondeterministically.
\item $\alpha^*$ repeats $\alpha$ zero or more times nondeterministically.
\item $x := \theta$ evaluates term $\theta$ and assigns result to $x$. 
\item $x := *$ assigns an arbitrary real value to $x$. 
\item $\{x_1'=\theta_1,...,x_n'=\theta_n \& F\}$ is the continuous evolution of $x_i$ along the solution to the system constrained to a domain defined by $F$.
\item $?F$ aborts if formula $F$ is not true.
\end{itemize}

Hybrid programs have a denotational semantics that defines, for reach program, the set of states that are reachable by executing the program from an initial state. A state is an assignment of variables to values. For example, the denotation of $x := t$ in a state $s$ is:
\begin{align*}
    \den{x:=t}(s)(v) &= s(v) \text{ for } v \not = x \\
    \den{x:=t}(s)(x) &= t
\end{align*}

Composite programs are given meaning by their constituent parts. For example, the meaning of $\alpha \cup \beta$ is:
\[
\den{\alpha \cup \beta}(s) = \den{\alpha}(s) \cup \den{\beta}(s)
\]

A full definition of the denotational semantics corresponding to the informal meanings given above is provided by \citep{DBLP:conf/cade/Platzer15}.

\subsection{Differential Dynamic Logic Overview}

Formulas of \dL are generated by the grammar:
\[
\varphi,\psi ::= f \sim g ~|~  \varphi \land \psi ~|~ \varphi \lor \psi ~|~ \varphi \limply \psi ~|~ \forall x, \varphi ~|~ \exists x, \varphi ~|~ \dibox{\alpha}\varphi
\]
where $f,g$ are polynomials of real arithmetic, $\sim$ is one of $\{ \le, <, =, >, \ge\}$, and the meaning of $\dibox{\alpha}\varphi$ is that $\varphi$ is true in every state that can be reached by executing the program $\alpha$. Formulas of \dL can be stated and proven in the KeYmaera~X theorem prover \citep{DBLP:conf/cade/FultonMQVP15,bellerophon}. 

The meaning of a \dL formula is given by a \emph{denotational semantics} that specifies the set of states $s \in S$ in which a formula is true. For example,
\begin{align*}
    \den{true} &= S \\
    \den{false} &= \emptyset \\
    \den{x=1 \land y=2} &= \{s \in S | s(x) = 1 \text{ and } s(y) = 2\}
\end{align*}
\noindent We write $\models\varphi$ as an alternative notation for the fact that $\varphi$ is true in all states (i.e., $\den{\varphi} = \den{true}$).
We denote by $\vdash \varphi$ the fact that there is a proof of $\varphi$ in the proof calculus of \dL.

\subsection{Using Safe Controller Specifications to Constrain Reinforcement Learning}
\newcommand{\CM}{CM}
\newcommand{\MM}{MM}

Given a hybrid program and proven \dL safety specification, \citet{aaai18} explains how to construct safety monitors (which we also call \emph{safe actions filters} in this paper) for reinforcement learning algorithms over a symbolic state space. In this section, we summarize their algorithm.

As opposed to our approach, \citet{aaai18} employs both a controller monitor (that ensures the safety of the controller) and a model monitor (that ensures the adherence of the model to the actual system and checks for model mismatch). 

The meaning of the controller monitor and model monitor are stated with respect to a specification with the syntactic form $P \limply \dibox{\{\text{ctrl};\text{plant}\}^*} Q$ where $P$ is a \dL formula specifying initial conditions, $\text{plant}$ is a dynamical system expressed as a hybrid program that accurately encodes the dynamics of the environment, and $Q$ is a post-condition. \citep{aaai18} assumes that $\text{ctrl}$ as the form $?P_1;a_1 \cup \cdots \cup P_n; a_n$, where $a_i$ are discrete assignment programs that correspond to the action space of the RL agent. For example, an agent that can either accelerate or brake as action space $A = \{ A, -B \}$. The corresponding control program will be $?P_1; a:=A \cup ?P_2; a:=-B$ where $P_1$ is a formula characterizing when it is safe to accelerate and $P_2$ is a formula characterizing when it is safe to brake.

Given such a formula, \citep{aaai18} defines the controller and model monitors using the following conditions:

\begin{cor}[Meaning of Controller Monitor] \label{cor:CM}
Suppose $\CM$ is a controller monitor for $P \limply \dibox{\{\text{ctrl};\text{plant}\}^*} Q$
and \m{s \in S} and $u : S \rightarrow S$.
Then $\CM(u,s)$ implies  $(s, u(s)) \in \den{\text{ctrl}}$.
\end{cor}

\begin{cor}[Meaning of Model Monitor] \label{cor:MM}
Suppose $\MM$ is a model monitor for $\text{init} \limply \dibox{\{\text{ctrl};\text{plant}\}^*} Q$,
that $u$ is a sequence of actions, and that $s$ is a sequence of states.
If
$\MM(s_{i-1}, u_{i-1}, s_{i})$ for all $i$
then
$s_i \models Q$, and also 
$(s_i, u_i(s_i)) \in \den{\text{ctrl}}$ implies $(u_i(s_i), s_{i+1}) \in \den{\text{plant}}$.
\end{cor}

\section{Proof of \rref{thm:safety}} \label{appendix:proof}

If the object detector produces an accurate mapping, then \rref{alg:mname} will preserve the safety constraint associated with 
the $\varphi$ monitor. We state this property formally in \rref{thm:correctness}.

\begin{theorem*}[Safety Theorem]\label{thm:safety}
Assume the following conditions hold along a trajectory $s_0, a_0, \dots, s_n$ with $s_0 \in \sinit$:
\begin{description}
    \item[A1] Initial states are safe: $s \in \sinit$ implies $\psi(s) \models \textit{init}$.
    \item[A2] The model and symbolic mapping are correct up to simulation: If $T(s_i,a,s_j) \not = 0$ for some action $a$ then 
    $(\psi(s_i), a(\psi(s_i))) \in \den{\textit{ctrl}}$ and  
    $(\psi(s_i), \psi(s_j)) \in {\small \den{\textit{plant}}}$.
\end{description}
\end{theorem*}
\begin{proof}
We begin the proof by pointing out that our assumption about how $$\textit{init} \limply \dibox{\{\textit{ctrl};\textit{plant}\}^*}\textit{safe}$$ was proven provides us with the following information about some formula $J$:

\begin{align*}
\vdash&~ \textit{init} \limply J  & (\textbf{LI1}) \\
\vdash&~ J \limply \textit{safe}  & (\textbf{LI2}) \\
\vdash&~ J \limply \dibox{\{\textit{ctrl};\textit{plant}\}^*}J  & (\textbf{LI3})
\end{align*}

Now, assume $s_0, a_0, s_1, a_1, \dots, s_n$ with $s_0 \in \sinit$ is a trajectory generated by running an RL agent with actions selected by \rref{alg:mname}
and proceed by induction on the length of the sequence with the inductive hypothesis that $\psi(s_i) \models J$.

If $i=0$ then $s_0 \in \sinit$ by assumption. Therefore, $\psi(s_0) \models \textit{init}$ by \textbf{A1}.
We know by \textbf{LI1} that $\vdash \textit{init} \limply J$.
Therefore, $\psi(s_0) \models J$ by Modus Ponens and the soundness of the \dL proof calculus.

Now, suppose $i > 0$. We know $\psi(s_i) \models J$ by induction.
Furthermore, we know $T(s_i, a_i, s_{i+1}) \not = 0$ because otherwise this trajectory could not exist.
By \textbf{A2} and the denotation of the $;$ operator, we know $(\psi(s_i), \psi(s_{i+1})) \in \den{\textit{ctrl};\textit{plant}}$.
By \textbf{LI3}, we know $\vdash J \limply \dibox{\textit{ctrl};\textit{plant}}J$
Therefore, $\psi(s_i) \models J$ and $(\psi(s_i), \psi(s_{i+1})) \in \den{\textit{ctrl};\textit{plant}}$ implies $\psi(s_i+1) \models J$
by the denotation of the box modality and the soundness of \dL.

We have now established that $\psi(s_i) \models J$ for all $i \ge 0$.
By \textbf{LI2}, Modus Ponens, and soundness of the \dL proof calculus, we finally conclude that $\psi(s_i) \models \textit{safe}$.
\end{proof}

Note that if all actions $a_i$ along the trajectory are generated using \rref{alg:mname}, and if the model is accurate, then the two assumptions in \rref{thm:correctness} will hold.

\section{Proof of \rref{thm:policy_equivalence}}
\label{app:env_wrapping}

In order to enforce safety, we wrap the original environment in a new one which has no unsafe actions. By not modifying the agent or training algorithm, any theoretical results (e.g. convergence) which the algorithm already has will still apply in our safety-wrapped environment. However, it is still necessary to show the relation between the (optimal) policies that may be found in the safe environment and the policies in the original environment. We show that 1) all safe policies in the original environment have the same transition probabilities and expected rewards in the wrapped environment and 2) all policies in the wrapped environment correspond to a policy in the original environment which has the same transition probabilities and expected rewards. This shows that the optimal policies in the wrapped environment are optimal among safe policies in the original environment (so no reward is lost except where required by safety).

Let the original environment be the MDP $E = (\mc S, \mc A, T, R)$. We define a safety checker $C : \mc S \times \mc A \rightarrow \{T, F\}$ to be a predicate such that $C(s, a)$ is True iff action $a$ is safe in state $s$ in $E$. When we refer to an action as safe or unsafe, we always mean in the original environment $E$. A policy $\pi$ in $E$ is safe iff
\[
\forall s \in \mc S \;\forall a \in \mc A \; \pi(a|s) > 0 \implies C(s, a).
\]

The safety-wrapped environment will be $E' = (\mc S, \mc A, T', R')$ where the transition and reward functions will be modified to ensure there are no unsafe actions and expected rewards in $E'$ correspond with those from acting safely in $E$.

$T'$ is required to prevent the agent from taking unsafe actions; for any safe action, we keep this identical to $T$. When an unsafe action is attempted, we could either take a particular safe action deterministically (perhaps shared across states, if some action is always safe, or a state-specific safe action) or sample (probably uniformly) from the safe actions in a given state. We prefer the latter approach of sampling from the safe actions because this makes taking an unsafe action have higher variance, so the agent will probably learn to avoid such actions. If unsafe actions are deterministically mapped to some safe action(s), they become indistinguishable, so the agent has no reason to avoid unsafe actions (unless we tamper with the reward function). Thus we set

\[
T'(s, a, s') = \begin{cases}
T(s, a, s') & \text{if } C(s, a) \\
\frac{1}{|\mc A_{C(s)}|} \sum_{a' \in \mc A_{C(s)}} T(s, a', s') & \text{otherwise}
\end{cases}
\]

where $\mc A_{C(s)} = \{a \in \mc A \mid C(s, a)\}$ is the set of safe actions in state $s$. This simulates replacing unsafe actions with a safe action chosen uniformly at random.

$R'$ is defined similarly so that it simulates the reward achieved by replacing unsafe actions with safe ones uniformly at random:

\[
R'(s, a) = \begin{cases}
R(s, a) & \text{if } C(s, a) \\
\frac{1}{|\mc A_{C(s)}|} \sum_{a' \in \mc A_{C(s)}} R(s, a') & \text{otherwise}
\end{cases}.
\]

\begin{lemma}
For every safe policy $\pi$ in E, following that policy in $E'$ leads to the same transitions with the same probabilities and gives the same expected rewards.
\end{lemma}

\begin{proof}
By definition of safety, $\pi$ has zero probability for any $(s, a)$ where $C(s, a)$ isn't true. Thus actions sampled from $\pi$ lead to transitions and rewards from the branch of $T'$ and $R'$ where they are identical to $T$ and $R$.
\end{proof}

\begin{lemma}\label{lem:isomorphic_policies}
For every policy $\pi'$ in $E'$ there exists a safe policy $\pi$ in $E$ such that $\pi'$ has the same transition probabilities and expected rewards in $E'$ as $\pi$ does in $E$.
\end{lemma}

\begin{proof}
For any $\pi'$ in $E'$, let $g(\pi') = \pi$ be defined such that

\[
\pi(a|s) = \begin{cases}
\pi'(a|s) + \frac{1}{|\mc A_{C(s)}|} \sum_{a' \in \overline{\mc A}_{C(s)} } \pi'(a'|s) & \text{if } C(s, a) \\
0 & \text{otherwise}
\end{cases}
\]

where $\overline{\mc A}_{C(s)} = \{a \in \mc A ~|~ \lnot C(s, a)\}$ is the set of unsafe actions in state $s$. This simulates evenly redistributing the the probability that $\pi'$ assigns to unsafe actions in $s$ among the safe actions. 

We show first that the transition probabilities of $\pi$ in $E$ and $\pi'$ in $E'$ are the same.

\begin{align*}
P_{\pi, E}(s' | s) &= \sum_{a \in \mc A} \pi(a|s) T(s, a, s') \\
&= \sum_{a \in \mc A_{C(s)} } \pi(a|s) T(s, a, s') + \sum_{a \in \overline{\mc A}_{C(s)}} \underbrace{\pi(a|s)}_{=0} T(s, a, s') \\
&= \sum_{a \in \mc A_{C(s)} } \left(\pi'(a|s) + \frac{1}{|\mc A_{C(s)}|} \sum_{a' \in \overline{\mc A}_{C(s)}} \pi'(s, a')\right) T(s, a, s') \\ 
&= \sum_{a \in \mc A_{C(s)} } \pi'(a|s)  T(s, a, s') + \sum_{a' \in \overline{\mc A}_{C(s)}} \pi'(s, a') \frac{1}{|\mc A_{C(s)}|}  \left( \sum_{a \in \mc A_{C(s)} } T(s, a, s') \right) \\
&= \sum_{a \in \mc A_{C(s)} } \pi'(a|s)  T'(s, a, s') +  \sum_{a \in \overline{\mc A}_{C(s)}} \pi'(a|s) \frac{1}{|\mc A_{C(s)}|}  \left( \sum_{a' \in \mc A_{C(s)} } T(s, a', s') \right) \\
&= \sum_{a \in \mc A_{C(s)} } \pi'(a|s)  T'(s, a, s') + \sum_{a \in \overline{\mc A}_{C(s)} } \pi'(a|s)  T'(s, a, s') \\
&= \sum_{a \in \mc A } \pi'(a|s)  T'(s, a, s') \\
&= P_{\pi', E'}(s' | s)
\end{align*}

Let $\E[\pi, E]{R_s}$ be the expected reward of following the policy $\pi$ in environment $E$ at state $s$. The equality of the expected reward for $\pi$ in every state of $E$ and $\pi'$ in every state of $E'$ can be shown similarly:

\begin{align*}
\E[\pi, E]{R_s} &= \sum_{a \in \mc A} \pi(a|s) R(s, a) \\
&= \sum_{a \in \mc A} R(s, a) \begin{cases}
\pi'(a|s) + \frac{1}{|\mc A_{C(s)}|} \sum_{a' \in \overline{\mc A}_{C(s)}} \pi'(a'|s) & \text{if } C(s, a) \\
0 & \text{otherwise}
\end{cases}\\
&= \sum_{a \in \mc A_{C(s)}} R(s, a) \left(\pi'(a|s) + \frac{1}{|\mc A_{C(s)}|} \sum_{a' \in \overline{\mc A}_{C(s)}} \pi'(a'|s)\right)\\
&= \left(\sum_{a \in \mc A_{C(s)}} \pi'(a|s) R(s, a)\right) + \left(\frac{1}{|\mc A_{C(s)}|} \sum_{a' \in \overline{\mc A}_{C(s)}} \pi'(a'|s)\right) \sum_{a \in \mc A_{C(s)}} R(s, a)
\end{align*}

\begin{align*}
\E[\pi', E']{R'_s} &= \sum_{a \in \mc A} \pi'(a|s) R'(s, a)\\
&= \sum_{a \in \mc A} \pi'(a|s) \begin{cases}
R(s, a) & \text{if } C(s, a)\\
\frac{1}{|\mc A_{C(s)}|} \sum_{a' \in \mc A_{C(s)}} R(s, a') & \text{otherwise}
\end{cases}\\
&= \left(\sum_{a \in \mc A_{C(s)}} \pi'(a|s) R(s, a)\right) + \left(\frac{1}{|\mc A_{C(s)}|} \sum_{a' \in \mc A_{C(s)}} R(s, a')\right) \sum_{a \in \overline{\mc A}_{C(s)}} \pi'(a|s)\\
&= \E[\pi, E]{R_s}
\end{align*}
\end{proof}

\begin{theorem*} 
Let $E$ be an environment and $L$ a reinforcement learning algorithm.
If $L$ converges to a reward-optimal policy $\pi^*$ in $E$,
then using \rref{alg:mname} with $L$ converges to $\pi^*_s$, the safe policy with the highest reward (i.e. the \emph{reward-optimal safe policy}).
\end{theorem*}

\begin{proof}
We provide proof by contraposition. Let's assume that $\pi^*$ is not optimal in $E'$. Then there must exist $\pi'$ in $E'$ that gets more reward. But, by Lemma \ref{lem:isomorphic_policies}, $\pi'$ corresponds to a safe policy $\pi = g(\pi')$ in $E$ which gets the same amount of reward, so $\pi$ is better in $E$ than $\pi^*$. Hence, $\pi^*$ is not optimal among safe policies in $E$.
\end{proof}

A few notes regarding this theorem:
\begin{itemize}
    \item The intuitive approach to making an agent safe, if we know the set of safe actions in each state, might be to sample from the safe subset of the agent's policy distribution (after renormalization). Because this is not actually sampling from the distribution the agent learned, this may interfere with training the agent.
    \item While we keep $\mc S$ the same in $E$ and $E'$, there may be states which become unreachable in $E'$ because only unsafe transitions in $E$ lead to them. Thus the effective size of $E'$'s state space may be smaller which could speed up learning effective safe policies.
    \item Our approach can be viewed as transforming a constrained optimization problem (being safe in $E$; have to treat it as a CMDP) into an unconstrained one (being safe in $E'$).
\end{itemize}

\section{Object Detection Details}
\label{app:symbolicMapping}

\paragraph{CenterNet \citep{zhou2019objects_centernet}}

CenterNet-style object detectors take an image of size $H \times W \times C$ (height, width, and channels, respectively) as input and output an image $Z$ of size $\lfloor H / S \rfloor \times \lfloor W / S \rfloor \times (N + 2)$ where $S$ is a downscaling amount to make the detection more efficient and $N$ is the number of classes to detect. For the first $N$ channels, $z_{ijk}$ is the probability that the pixel $i, j$ is the (downscaled) center of an object of the $k$th class. The final two channels of $Z$ contain x and y offsets. The offsets account for the error in detecting locations in the original image because the predictions are downscaled: a downscaled detection at $i, j$ can be converted to a detection in the original image coordinates at $i', j'$ by setting $i' = i * S + o_1, j' = j * S + o_2$ where $o_1 = z_{ijN}, o_2 = z_{ij(N+1)}$. As the objects in our environments have constant sizes, we don't predict the object sizes as is done in CenterNet.

As in \cite{zhou2019objects_centernet}, we set $S = 4$ and use max-pooling and thresholding to convert from the probability maps to a list of detections. In particular, there is a detection for object class $k$ at location $i,j$ if $z_{ijk} \ge \tau$ and $z_{ijk} == maxpool(z_{ijk})$ where $maxpool$ is a 3x3 max-pooling operation centered at $i,j$ (with zero-padding). We set $\tau = 0.5$. The detector then returns a list of tuples $(k, i', j')$ containing the class id ($k$) and center point ($i', j'$) of each detection. These are used in evaluating the constraints wherever the formulas reference the location of an object of type $k$ (i.e. if a robot must avoid hazards, the constraint will be checked using the location of each hazard in the detections list).

We use ResNet-18 \citep{he2016deep_resnet} truncated to the end of the first residual block. The first layer is also modified to have only a single input channel because we use grayscale images, as is common for inputs to RL agents. This already outputs an image which is downscaled 4x relative to the input, so we do the centerpoint classification and offset prediction directly from this image, removing the need for upscaling. We use one 1x1 convolutional layer for the offset prediction (two output channels) and one for the center point classification ($N$ output channels and sigmoid activation).

\paragraph{Training}

To avoid introducing a dependency on heavy annotations, we restricted ourselves to a single image for each safety-relevant object in an environment and a background image. We produce images for training by pasting the objects into random locations in the background image. We also use other standard augmentations such as left-right flips and rotations. New images are generated for every batch.

We use the label-generation and loss function from \cite{zhou2019objects_centernet}. Labels for each object class are generated by evaluating, at each pixel position, a Gaussian density on the distance from that position to the center of the nearest object of the given class (see \rref{alg:label_creation} for details).

\begin{algorithm}
\caption{Label Creation}
\label{alg:label_creation}
\begin{algorithmic}
\STATE \textbf{Input} $xs, ys$: center $x, y$ positions for each object of a type; $h, w$: label image height, width
\STATE $Y \leftarrow 0_{h \times w}$ \text{  // an array of zeros of size $h \times w$}
\STATE $\Sigma \leftarrow \left[\begin{array}{cc}
    h / 2 & 0  \\
    0 & w / 2 
\end{array}\right]$\;
\FOR {$x \in xs, y \in ys$}
    \STATE $\mu \leftarrow [x, y]$\;
    \FOR{i \assign 1, 2, \dots, h}
      \FOR{j \assign 1, 2, \dots, w}
        \STATE \text{// $\phi_{\mu, \Sigma}$ is the probability density of a multivariate Gaussian parametrized by $\mu$ and $\Sigma$}
        \STATE $Y_{ij} \leftarrow \max(Y_{ij}, \phi_{\mu, \Sigma}([i, j]))$ \;
      \ENDFOR
    \ENDFOR
\ENDFOR
\STATE \textbf{return} Y\;
\end{algorithmic}
\end{algorithm}

The loss function is a focal loss: a variant of cross-entropy that focuses more on difficult examples (where the predicted probabilities are farther from the true probabilities) \citep{lin2017focal}. We use a modified focal loss as in \citep{law2018cornernet,zhou2019objects_centernet}:

\begin{align*}
    \frac{-1}{N} \sum_{i, j, k} \begin{cases} 
      (1 - \hat Y_{ijk})^\alpha \log(\hat Y_{ijk}) & \text{if } Y_{ijk} = 1 \\
      (1 - Y_{ijk})^\beta \hat Y_{ijk}^\alpha \log(1 - \hat Y_{ijk}) & \text{otherwise} 
  \end{cases}
\end{align*}

\noindent where $N$ is the number of objects in the image (of any type); $i \in [1, h]$; $j \in [1, w]$; $t \in [1, T]$; $w, h$ are the width and height of the image; and $T$ is the number of object classes. $\hat Y_{ijk}$ is the predicted probability of an object of type $t$ being centered at position $(x, y)$ in the image and $Y_{ijk}$ is the ``true'' probability. $\alpha, \beta$ are hyperparameters that we set to 2 and 4, respectively, as done by \citep{law2018cornernet,zhou2019objects_centernet}. We remove the division by $N$ if an image has no objects present. The loss for the offsets is mean-squared error, and we weight the focal loss and offset loss equally. We use the Adam optimizer \citep{kingma2014adam} with learning rate 0.000125, $beta_1 = 0.9$, $\beta_2 = 0.999$, as in \cite{zhou2019objects_centernet}. We decrease the learning rate by a factor of 10 if the loss on a validation set of 5,000 new images doesn't improve within 10 training epochs of 20,000 images. The batch size is 32 as in \cite{zhou2019objects_centernet}. We keep the model which had the best validation loss.

\section{Reinforcement Learning Details} \label{appendix:fullcode}

In all of our experiments, we use PPO as the reinforcement learning algorithm. Our hyperparameter settings are listed in Table \ref{tab:ppo_params}. We run several environments in parallel to increase training efficiency using the method and implementation from \citep{stooke2019rlpyt}.

We use grayscale images as inputs to the RL agent, and the CNN architecture from \cite{espeholt2018impala}.

\begin{table}[hb!]
    \centering
    \begin{tabular}{l|l}
        Hyperparameter & Value \\
        \midrule
        Adam learning rate & $0.001 \times \alpha$ \\
        Num. epochs & 4 \\
        Number of actors & 32 \\
        Horizon (T) & 64 \\
        Minibatch size & $2048 ~(=32 \times 64)$ \\
        Discount ($\gamma$) & 0.99 \\
        GAE parameter ($\lambda$) & 0.98 \\
        Clipping parameter & $0.1 \times \alpha$ \\
        Value function coeff. & 1 \\
        Entropy coeff. & 0.01 \\
        Gradient norm clip & 1 \\
    \end{tabular}
    \caption{Hyperparameters used for PPO in all experiments. $\alpha$ is linearly annealed from 1 at the start of each experiment to 0 at the end.}
    \label{tab:ppo_params}
\end{table}

\end{document}